\documentclass[11pt]{article}
\usepackage{amsmath}
\usepackage{amsfonts}
\usepackage{amssymb}
\usepackage{color}
\usepackage{latexsym}
\usepackage{multirow}
\usepackage{graphicx}
\newtheorem{theorem}{\bf Theorem}[section]
\newtheorem{lemma}[theorem]{\bf Lemma}
\newtheorem{prop}[theorem]{\bf Proposition}

\newtheorem{coro}[theorem]{\bf Corollary}

\newtheorem{defn}[theorem]{\bf Definition}
\newenvironment{proof}{\noindent{\em Proof:}}{\quad \hfill$\Box$\vspace{2ex}}

\def \bI {\Bbb I}
\def \bN {\Bbb N}

\def \bR {\Bbb R}

\def \bY {\Bbb Y}

\def \bx {{\bf x}}
\def \by {{\bf y}}
\def \ba {{\bf a}}
\def \ca {{\bf a}}

\def \bc {{\bf c}}
\def \cc {{\bf c}}
\def \cd {{\bf d}}
\def \bd {{\bf d}}

\def \bz {{\bf z}}

\def \cH {{\cal H}}

\def \cE {{\cal E}}

\def \and {\, \mbox{\rm and}\, }

\makeatletter

\newcommand{\Rmnum}[1]{\expandafter\@slowromancap\romannumeral #1@}
\makeatother
\begin{document}
\title{\bf Learning Rates for Multi-task Regularization Networks\thanks{Supported in part by National Natural Science Foundation of China under grant 11971490, and by Natural Science Foundation of Guangdong Province under grant 2018A030313841.}}
\author{Jie Gui\thanks{School of Computer Science and Engineering, Sun Yat-sen University, Guangzhou 510006, P. R. China. E-mail address: {\it guij6@mail2.sysu.edu.cn}.}\quad and\quad Haizhang Zhang\thanks{Corresponding author. School of Mathematics (Zhuhai), and Guangdong Province Key Laboratory of Computational Science, Sun Yat-sen University, Zhuhai 519082, P. R. China. E-mail address: {\it zhhaizh2@mail.sysu.edu.cn}. }}
\date{}
\maketitle
\begin{abstract}
Multi-task learning is an important trend of machine learning in facing the era of artificial intelligence and big data. Despite a large amount of researches on learning rate estimates of various single-task machine learning algorithms, there is little parallel work for multi-task learning. We present mathematical analysis on the learning rate estimate of multi-task learning based on the theory of vector-valued reproducing kernel Hilbert spaces and matrix-valued reproducing kernels. For the typical multi-task regularization networks, an explicit learning rate dependent both on the number of sample data and the number of tasks is obtained. It reveals that the generalization ability of multi-task learning algorithms is indeed affected as the number of tasks increases.

\end{abstract}

\section{Introduction}
\setcounter{equation}{0}

Machine learning designs algorithms so that computers can learn certain intelligent behaviors from finite sample data. The accuracy of the prediction function learned by an algorithm on new input data is called the generalization ability of the algorithm. The generalization ability is quantified by certain distance between the prediction function and the optimal function in mathematics. The rate at which the distance converges to zero as the number of data increasing to infinity is called the learning rate of the algorithm. Estimation of learning rates is a crucial question in mathematical study of machine learning.

The mathematical foundation of learning rate estimates was built by Cucker, Smale, Zhou and their collaborators \cite{CuckerSmale,CZ07}.  There has been an extensive collection of studies on learning rates for single-task machine learning methods (see, for example, \cite{ChenPanLi,GuoShi,HuangChen,SmaleZhou,SmaleZhou2,SongZhang,SunWu,TongChen}, and the references therein). The analysis is based on the theory of scalar-valued reproducing kernels and scalar-valued Reproducing Kernel Hilbert Spaces (RKHS) \cite{Aronszajn,ScSm}.

In facing the ear of big data, multi-task learning where the unknown target function is vector-valued appears more often in applications. Micchelli and Pontil proposed to study mathematics of multi-tasking learning based on operator-valued reproducing kernels and vector-valued RKHS \cite{EMP,MP2005}. Vector-valued RKHS were discovered and studied far earlier to the rising of machine learning, \cite{BM,Pedrick}. Since the initial studies \cite{EMP,MP2005}, much work has been devoted to the theory of operator-valued reproducing kernels and vector-valued RKHS. For example, general theory of vector-valued RKHS was extensively discussed in \cite{CDT,CDTU}; universal multi-task kernels were characterized in \cite{CMPY,CDTU}; different forms of the Mercer theorem for operator-valued reproducing kernels were established in \cite{CDT,DVS}; inclusion relations of vector-valued RKHS was systematically investigated in \cite{ZXZ}. The more general functional RKHS was studied in \cite{WangXu}. Even vector-valued reproducing kernel Banach spaces \cite{LSZ,ZZ13} has been defined and constructed.

Surprisingly, as far as we know, unlike the fruitful results on learning rate estimates of single-task learning based on scalar-valued RKHS, there has been little parallel work for multi-task learning based on vector-valued RKHS. This is the motivation of the paper. We desire to build necessary foundational mathematical results for learning rate analysis of multi-task learning. In other words, we want to see how far the theory of Cucker, Smale, and Zhou can be extended to multi-task learning. The other purpose is to see how the number of tasks would affect the learning rate. We target at the classical regularization networks \cite{CuckerSmale,EPP}. Many results to be presented will be useful for learning rate estimates of other multi-task learning algorithms.

The rest of the paper is organized as follows. In section 2, we introduce the setting of multi-task learning based on matrix-valued reproducing kernels and vector-valued RKHS. In section 3, we present the important Mercer theorem which yields a characterization of vector-valued RKHS. Finally, we estimate the learning rate of multi-task regularization networks. An explicit upper bound of the learning rate will be given, which shows that the generalization ability of the learning algorithm is indeed affected as the number of tasks increases.

\section{Multi-task Learning with Matrix-valued Reproducing Kernels}
\setcounter{equation}{0}

We start with the standard setting of mathematical learning theory. Let $X$ be a compact metric space modelling the input space of data and $Y=\bR^m$ be the corresponding output space. A finite sample data set $\bz:=\{(x_i,y_i)\in X\times Y:1\le i\le n\}$ drawn independently and identically distributed (i.i.d.) from an unknown probability measure $\rho$ on $X\times Y$ is available. Let us keep in mind the two important constants $n$ and $m$, which stand for the number of data and the number of tasks, respectively. The main goal of machine learning is to infer from the sample data a prediction function from $X$ to $Y$ that yields satisfactory outputs for new inputs.

In order to measure the accuracy of a candidate prediction function, we introduce some norms and function spaces. Vectors in $Y=\bR^m$ are always viewed as $m\times 1$ column vectors. Two kinds of norms on $Y$ will be used:
$$
\|y\|_2:=\bigl(\sum_{j=1}^m|y_j|^2\bigr)^{1/2}\mbox{ and }\|y\|_\infty:=\max_{1\le j\le m}|y_j|,\ y=(y_j:1\le j\le n)^T\in Y.
$$
The standard Euclidean norm $\|\cdot\|_2$ is induced from the inner product on $Y$:
$$
\langle \xi,\eta \rangle_{2}:=\eta^T\xi,\ \ \xi,\eta\in Y,
$$
where $\eta^T$ denotes the transpose of $\eta$. It also induces a matrix norm on all $m\times m$ matrices:
$$
	\left\|A\right\|_2=\max\limits_{y\neq \textbf{0} \atop y \in \bR^m}\frac{\left\|Ay\right\|_2}{\left\|y\right\|_2}, \ A\in\bR^{m\times m}.
	$$
Let $\rho_X$ be the marginal probability measure of $\rho$ on $X$. We denote by $L^2_{\rho}(X,Y)$ the space of all square-integrable functions from $X$ to $Y$ with respect to $\rho_X$. It is a Hilbert space with the inner product and norm
$$
(f,g)_\rho=\int_X f(x)^Tg(x)d\rho_X(x),\ \ \|f\|_\rho=\biggl(\int_X f(x)^Tf(x)d\rho_X(x)\biggr)^{1/2}.
$$

The generalization ability of a candidate prediction function $f:X\to Y$ is measured by
$$
\cE(f):=\int_{X\times Y}\|f(x)-y\|_2^2d\rho.
$$
The optimal function that minimizes this error is
$$
f_\rho(x):=\int_Y yd\rho(y|x), \ \ x\in X.
$$
Here $d\rho(y|x)$ is the conditional measure of $y$ with respect to $x$. Similar to the scalar-valued case in \cite{CuckerSmale}, we have for each $f\in L^2_{\rho}(X,Y)$
\begin{equation}\label{optimalrho}
\cE(f)-\cE(f_\rho)=\|f-f_\rho\|_{\rho}^2=\int_X\|f(x)-f_\rho(x)\|^2_2d\rho_X(x).
\end{equation}
In fact,
$$
\begin{aligned}
\cE(f)&=\int_{X\times Y}\|f(x)-y\|_2^2d\rho=\int_{X\times Y}\|f(x)-f_\rho(x)+f_\rho(x)-y\|_2^2d\rho \\
&=\int_{X\times Y}\|f(x)-f_\rho(x)\|^2_2d\rho+\int_{X\times Y}\|f_\rho(x)-y\|_2^2d\rho+2\int_{X\times Y}(f(x)-f_\rho(x))^T(f_\rho(x)-y)d\rho\\
& =\int_X\|f(x)-f_\rho(x)\|_2^2d\rho_X(x)\int_Yd\rho(y|x)+\cE(f_\rho)+
\int_X(f(x)-f_\rho(x))^Td\rho_X(x)\int_Y (f_\rho (x)-y)d\rho(y|x)\\
&=\int_X\|f(x)-f_\rho(x)\|^2_2d\rho_X(x)+\cE(f_\rho)+\int_X(f(x)-f_\rho(x))^T{\bf 0}d\rho_X(x)\\
&=\|f-f_\rho\|_{\rho_X}^2+\cE(f_\rho).\\
\end{aligned}
$$
However, as the probability measure $\rho$ is unknown, the optimal function $f_\rho$ is intractable. By (\ref{optimalrho}), we desire to learn a prediction function $f_\bz$ from the finite sample data $\bz$ so that $\|f_\bz-f_\rho\|_{\rho}$ is small.

To fulfill the task, we consider the classical regularization networks
\begin{equation}\label{mul-regnetworks}
 f_{\bz,\lambda}:=\arg\min_{f\in \cH_{K}}\frac{1}{n}\sum_{i=1}^n \left\|f(x_i)-y_i\right\|_2^2+\lambda\left\|f\right\|_{K}^2.
\end{equation}
Here $\lambda>0$ is a regularization parameter, $K$ is matrix-valued reproducing kernel on $X$, and $\cH_K$ is the corresponding vector-valued RKHS. We explain the definitions and notations in details below.

For simplicity, we denote by $L(Y,Y)$ the space of all bounded linear operators on $Y=\bR^m$. It coincides with the set of all $m\times m$ real matrices.

\begin{defn}\label{matrixkernel}
     We call a function $K: X \times X \to L(Y,Y)$ a {\bf matrix-valued reproducing kernel} on $X$ if it is {\bf symmetric} in the sense that
     $$
     K(x,x')=K(x',x)^T,\ \ \forall\ x,x'\in X
     $$
     and is {\bf positive-definite} in the sense that for all
     $x_1,x_2,\cdots,x_p \in X$ and all $\xi_i \in Y, 1 \leq i \leq p$, $p\in\bN$, it holds
     $$
     \sum_{i=1}^{p}\sum_{j=1}^{p} \langle K(x_i,x_j)\xi_i,\xi_j\rangle_2  \ge 0.
     $$
\end{defn}

\begin{defn}\label{vectorRKHS}
     A {\bf vector-valued reproducing kernel Hilbert space (RKHS)} on $X$ is a Hilbert space $\cH$ of certain functions from $X$ to $Y$ such that for each $x\in X$,
     $$
	\delta_x(f) :=  f(x), \ \ f\in \cH
	$$
is a continuous linear operator from $\cH$ to $Y$.
\end{defn}

A matrix-valued reproducing kernel $K$ on $X$ corresponds to a unique vector-valued RKHS on $X$, which we denote as $\cH_K$. The inner product and norm on $\cH_K$ is denote by $\langle\cdot,\cdot\rangle_K$ and $\|\cdot\|_K$, respectively. There are some important properties \cite{MP2005,Pedrick} of $K$ and $\cH_K$ that will be frequently used in later discussion.
\begin{enumerate}
\item For each $x\in X$, $K(x,x)$ is a positive-definite matrix, that is,
$$
\langle K(x,x)\xi,\xi\rangle_2\ge0,\ \ \xi\in Y.
$$

\item For all $x\in X$ and $\xi\in Y$, $K(x,\cdot)\xi\in\cH_K$ and there holds the reproducing identity:
\begin{equation}\label{riesz-represent-theorem}
	\langle f(x),\xi \rangle _2 =\langle f,K(x,\cdot)\xi \rangle_K, \ \ \forall f\in \cH_K,x\in X,\xi \in Y.
\end{equation}

\item The linear span of $\{K(x,\cdot)\xi:x\in X,\ \xi\in Y\}$ is dense in $\cH_K$.

\item For all $f\in\cH_K$ and $x\in X$, it holds
\begin{equation}\label{normofpoint}
\|f(x)\|_2\le \sqrt{\|K(x,x)\|_2}\|f\|_K.
\end{equation}
\end{enumerate}

The following representer theorem on the minimizer of (\ref{mul-regnetworks}) is well-known \cite{AMP2009,EMP,MP2005}.

\begin{lemma}\label{representertheorem}
The minimizer $f_{\bz,\lambda}$ of the regularization networks algorithm (\ref{mul-regnetworks}) exists and is unique. Moreover, there exist $c_i\in Y,1\leq i \leq n$ such that
$$
f_{\bz,\lambda}=\sum_{i=1}^{n}K(x_i,\cdot)c_i.
$$
\end{lemma}
 We need to find the coefficients $c_i$'s in the above equation for $f_{\bz,\lambda}$. To this end, we introduce the following sampling operator and its adjoint.

\begin{defn}\label{sample}
	{\bf  (vector-valued sampling operator)}   Given a finite set of sampling points $\bx=\{x_i\}_{i=1}^{n}$, we define the sampling operator $S_\bx:\cH_K \to Y^{n}$ by
	$$
	S_\bx f:=f(\bx), \ \ \forall f\in \cH_K,
	$$
	where $f(\bx)=(f(x_i):1 \leq i \leq n)\in Y^{n}.$
\end{defn}
The inner product on the tensor-product space $Y^n$ is
$$
\langle \bc,\bd\rangle_{Y^n}:=\sum_{i=1}^n \langle c_i,d_i\rangle_{2},\ \ \bc=(c_i:1\le i\le n)\in\bY^n,\ \bd=(d_i:1\le i\le n)\in Y^n.
$$
Thus,
$$
\langle S_{\bx}f,\bc \rangle_{Y^n} =\langle f(x),\bc \rangle_{Y^n}=\sum_{i=1}^{n}\langle f(x_i),c_i \rangle_Y=\langle f, \sum_{i=1}^{n}K(x_i,\cdot)c_i \rangle_K,
$$
which implies that the adjoint operator $S_\bx^*$ of $S_\bx$ is
$$
S_{\bx}^* \bc=\sum_{i=1}^{n}K(x_i,\cdot)c_i, \ \ \bc=(c_i:1\leq i \leq n) \in Y^n.
$$

We are now ready to derive an explicit expression of $f_{\bz,\lambda}$. Denote for each $f\in\cH_K$
\begin{equation}\label{mul-twoerrors}
E_{\bz,\lambda}(f)=\frac{1}{n}\sum_{i=1}^n\left\|f(x_i)-y_i\right\|_{2}^2+\lambda\left\|f\right\|_{K}^2.
\end{equation}
\begin{theorem}\label{obvious-answer}
	The minimizer of (\ref{mul-regnetworks}) has the expression
	\begin{equation}\label{representerequ}
	 f_{\bz,\lambda}=(\frac{1}{n}S_{\bx}^* S_{\bx}  +\lambda I )^{-1}\frac{1}{n}S_{\bx}^* \by,
	\end{equation}
	where $\by:=(y_i:1\le i\le n)\in Y^n$ and $I$ is the identity operator on $\cH_{K}$.
\end{theorem}
\begin{proof}
By Lemma \ref{representertheorem}, there exists $\bc\in Y^n$ such that the minimizer of (\ref{mul-regnetworks}) has the form
$$
f_{\bz,\lambda}=\sum_{i=1}^{n}K(x_i,\cdot)c_i=S_{\bx}^* \bc.
$$
Therefore, we desire to find the minimizer of (\ref{mul-regnetworks}) among functions of the form
$$
f=S_{\bx}^* \cd,\ \cd=(d_i:1\le i\le n)\in Y^n.
$$
Substitute this form into (\ref{mul-twoerrors}) to get
$$
\begin{aligned}
E_{\bz,\lambda}(f)&=\frac{1}{n}\left\| f(x) -{\by}\right\|_{Y^n}^2+\lambda \langle f,f\rangle_{K} \\
&=\frac{1}{n}\left\|S_{\bx}f -{\by}\right\|_{Y^n}^2+\lambda \langle S_{\bx}^* \cd,S_{\bx}^* \cd\rangle _{K} \\
&=\frac{1}{n}\left\| A\cd -{\by}\right\|_{Y^n}^2+\lambda \langle A\cd, \cd\rangle _{Y^n} \\
:&=E(\cd),
\end{aligned}
$$
where $A:=S_\bx S_\bx^*$. Since $K$ is positive-definite and
$$
\cd^TA\cd=\sum_{i=1}^n\sum_{j=1}^n \langle K(x_i,x_j)d_i,d_j\rangle_2,\ \ \cd=(d_i:1\le i\le n)\in Y^n,
$$
the matrix $A$ is symmetric and positive definite. As a consequence, for all $\lambda>0$,
\begin{equation}\label{obvious-answereq3}
\frac{1}{n}(A\cc-\by)+\lambda \cc=0
\end{equation}
possesses a unique solution
\begin{equation}\label{obvious-answereq4}
\cc=(\frac{1}{n}S_{\bx}S_{\bx}^{*}+\lambda J)^{-1}\frac{1}{n}\by,
\end{equation}
where $J$ is the $nm\times nm$ identity matrix. We next prove that this solution $\cc$ satisfies
$$
\cc=\arg\min_{\cd\in Y^n} E(\cd).
$$
To this end, it suffices to show that
$$
E(\cc+\cd)\ge E(\cc)\mbox{ for all }\cd\in Y^n
$$
as $Y^n$ is a vector space. We compute
$$
\begin{aligned}
E(\cc+\cd)&=\frac{1}{n}\left\|A(\cc+\cd)-\by\right\|_{Y^n}^2+\lambda \langle A(\cc+\cd),\cc+\cd \rangle_{Y^n} \\
&=\frac{1}{n}\left\|A\cc-\by\right\|_{Y^n}^2+\frac{1}{n}\left\|A\cd \right\|_{Y^n}^2+\lambda\langle A\cc,\cc\rangle_{Y^n}+\lambda\langle A\cd,\cd\rangle_{Y^n}\\
&\quad+\frac2n\langle A\cc-\by,A\cd\rangle_{Y^n}+2\lambda\langle A\cc,\cd\rangle_{Y^n}.\\
\end{aligned}
$$
By (\ref{obvious-answereq3}) and by that $A$ is symmetric,
$$
\begin{aligned}
\frac2n\langle A\cc-\by,A\cd\rangle_{Y^n}+2\lambda\langle A\cc,\cd\rangle_{Y^n}&=\frac2n\langle A\cc-\by,A\cd\rangle_{Y^n}+2\lambda\langle \cc,A\cd\rangle_{Y^n}\\
&=2\langle \frac{1}{n}(A\cc-\by)+\lambda \cc,A\cd\rangle_{Y^n}\\
&=0.
\end{aligned}
$$
Combining the above two equations, we get by the positive-definiteness of $A$ that
$$
E(\cc+\cd)\ge \frac{1}{n}\left\|A\cc-\by\right\|_{Y^n}^2+\lambda\langle A\cc,\cc\rangle_{Y^n}=E(\cc)\mbox{ for all } \cd\in Y^n,
$$
which confirms that $\cc$ given by (\ref{obvious-answereq4}) is a minimizer of $E(\cd)$.

We conclude that the minimizer $f_{\bz,\lambda}$ of (\ref{mul-regnetworks}) is given by
\begin{equation}\label{mul-explicitsolution}
    f_{\bz,\lambda}=S_\bx^*c=S_\bx^*(\frac{1}{n}S_{\bx}S_{\bx}^{*}+\lambda J)^{-1}\frac{1}{n}\by.
\end{equation}
Notice
$$
	S_{\bx}^*(\frac{1}{n}S_{\bx}S_{\bx}^{*}+\lambda J)=(\frac{1}{n}S_{\bx}^{*}S_{\bx}+\lambda I)S_{\bx}^*.
$$
It follows
$$
S_\bx^*(\frac{1}{n}S_{\bx}S_{\bx}^{*}+\lambda J)^{-1}=(\frac{1}{n}S_{\bx}^* S_{\bx}  +\lambda I )^{-1}S_{\bx}^*.
$$
Combining the above equation with (\ref{mul-explicitsolution}) yields
$$
f_{\bz,\lambda}=S_\bx^*(\frac{1}{n}S_{\bx}S_{\bx}^{*}+\lambda J)^{-1}\frac{1}{n}\by=
    (\frac{1}{n}S_{\bx}^* S_{\bx}  +\lambda I )^{-1}\frac{1}{n}S_{\bx}^* \by,
$$	
which completes the proof.
\end{proof}

We now have an explicit expression of $f_{\bz,\lambda}$. To estimate its distance to the optimal predictor $f_\rho$, we need to understand the vector-valued RKHS $\cH_K$. This will be done by the Mercer theorem via the integral operator on $L^2_{\rho}(X,Y)$ with kernel $K$.

\section{Characterizing Vector-valued RKHS by Mercer's Theorem}
\setcounter{equation}{0}

Recall that $X$ is a compact metric space, $Y=\bR^m$ and $K:X\times X\to L(Y,Y)$ is a continuous matrix-valued reproducing kernel on $X$. Also recall that $\rho_X$ is the marginal probability measure of $\rho$ on $X$ and $L^2_\rho(X,Y)$ denotes the space of all square-integrable functions from $X$ to $Y$ with respect to $\rho_X$. The inner product and norm on $L^2_{\rho}(X,Y)$ are
$$
(f,g)_\rho=\int_X f(x)^Tg(x)d\rho_X(x),\ \ \|f\|_\rho=\biggl(\int_X f(x)^Tf(x)d\rho_X(x)\biggr)^{1/2}.
$$
The vector-valued RKHS $\cH_K$ of $K$ will be characterized by the integral operator
\begin{equation}\label{integral}
	(L_K f)(x):=\int_{X}K(x,x')f(x')d\rho_{X}(x'),\ \ x \in X, \ \ f \in L^{2}_{\rho}(X,Y).
	\end{equation}
A few well-known properties \cite{CDT,CDTU} of $L_K$ will be needed:
\begin{enumerate}
\item For each $f\in L^{2}_{\rho}(X,Y)$, $L_K(f)$ lies in $ C(X,Y)$, the space of continuous functions from $X$ to $Y$.

\item The operator $L_K$ is compact on $L^{2}_{\rho}(X,Y)$. It is also self-adjoint and positive, that is,
$$
\langle L_K f,g\rangle_\rho=\langle f,L_Kg\rangle_\rho,\ \ \langle L_K f,f\rangle_\rho\ge 0,\ \ f,g\in L^{2}_{\rho}(X,Y).
$$
\end{enumerate}
Thus, all eigenvalues of $L_K$ are nonnegative. We remark that the case when $L_K$ has only finitely many nonzero eigenvalues is easier to handle and is hence not considered in the paper. Avoiding this case, we obtain by the theory of compact operators in functional analysis pairs of eigenfunctions and eigenvalues $(\phi_n,\lambda_n)$, $n\in\bN$ of $L_K$ such that
 $$
 L_K\phi_n=\lambda_n\phi_n,\ \ \lambda_n\ge \lambda_{n+1}>0, n\in\bN, \ \mbox{ and }\lim_{n\to\infty}\lambda_n=0.
 $$
 Moreover, $\{\phi_n:n\in\bN\}$ is an orthonormal sequence in $L^2_\rho(X,Y)$ and there holds
 \begin{equation}\label{orthogonalg}
 L_Kg=0\mbox{ for all }g\in L^2_\rho(X,Y)\mbox{ that is orthogonal to every }\phi_n,\ n\in\bN.
 \end{equation}

The celebrated Mercer's theorem \cite{CDT,DVS,Sun}, which states that $K$ can be expressed as a series in terms of the eigenfunctions and eigenvalues of $L_K$, plays an important role in learning theory. We shall need the following form of the Mercer theorem. It is worthwhile to point out the measure $\rho_X$ needs to be {\it non-degenerated} in order for the theorem to hold true. This is sometimes neglected in some references.

\begin{defn}
	{\bf  (non-degenerated measures)}   A positive Borel measure $\mu$ on a metric space $X$ is said to be non-degenerated if for every nonempty open subset $U\subseteq X$, $\mu(U)>0$.
\end{defn}

\begin{lemma}\label{mercer}\cite{CDT,DVS}
Let $X$ be compact metric space, $K$ be continuous matrix-valued reproducing kernel on $X$, and $\rho_X$ be non-degenerated on $X$. Suppose $(\phi_n,\lambda_n)$, $n\in\bN$ are the eigenfunctions and eigenvalues of $L_K$. Then
\begin{equation}\label{mercereq}
	K(x,y)=\sum_{n\in \bN}\lambda_{n}\phi_n(x)\phi_{n}^{T}(y), \ \ x,y\in X,
	\end{equation}
where the series converges uniformly and absolutely on $X\times X$.
\end{lemma}

We need some more definitions in order to characterize $\cH_K$ by the Mercer theorem.

\begin{defn}\label{square-root-operator}
	For $r>0$, denote by $L_K^{r}$ the linear operator on $L_{\rho}^2(X,Y)$ determined by
	$$
	L_K^{r}\phi_n=\lambda_{n}^{r}\phi_n, \ \ \forall n\in\bN
	$$
	and
	$$
	L_K^{r}g=0,\ \ \text{if }g\bot\phi_n,\ \forall n\in \bN.
	$$
In particular, $L_K^{1/2}$ is called the {\bf square-root operator} of $L_K$.

We also denote by $P_{\Phi}$ the {\bf orthogonal projection} of $L_{\rho}^2(X,Y)$ onto the closed subspace $\overline{span}\{\phi_{n}:n\in \bN \}$, where the closure is taken with respect to the norm on $L_{\rho}^2(X,Y)$.
\end{defn}

\begin{theorem}\label{Mercer-therom-result}
	Let $X$ be a compact metric space, $K$ be a continuous matrix-valued reproducing kernel on $X$, and $\rho_X$ be non-degenerated on $X$. Suppose $(\phi_n,\lambda_n)$, $n\in\bN$ are the eigenfunctions and eigenvalues of $L_K$. Then
	$$
	\mathcal{H}_{K}= L_K^{1/2}(L_{\rho}^2(X,Y))=\biggl\{f_c=\sum_{j\in \bN}c_j\phi_j :\sum_{j\in \bN}\frac{|c_j|^{2}}{\lambda_j}< + \infty\biggr\}
	$$
	and
$$
	\left\|L_K^{1/2}(f)\right\|_{K}=\left\|P_{\Phi}f\right\|_{\rho}, \ \ f\in L_{\rho}^2(X,Y).
	$$
\end{theorem}
\begin{proof}
	We decompose each $f\in L_{\rho}^2(X,Y)$ as
	$$
	f=\sum_{j\in \bN} \tilde{c}_j \phi_{j}+g,
	$$
	where $\tilde{c}_j=\langle f,\phi_j \rangle_{\rho}$ and $g\bot\phi_{j}$ for each $j\in \bN$. By the definition of the square-root operator of $L_K$, we have
	$$
	L_K^{1/2}f=\sum_{j\in \bN}\lambda_{j}^{1/2}\tilde{c}_j\phi_{j}.
	$$
	We rewrite it as
   \begin{equation}\label{functionfc}
	L_K^{1/2}f=f_{c}=\sum_{j\in \bI}c_j\phi_{j},
	\end{equation}
	where $c_j=\lambda_{j}^{1/2}\tilde{c}_j$. Note that $(\tilde{c}_j:j\in \bN)\in \ell^2(\bN) $.  Therefore, every function in
	$$
	\cH:=L_K^{1/2}(L_{\rho}^2(X,Y))
	$$
has the form
$$
f_c=\sum_{j\in\bN}c_j\phi_j\mbox{ with }\sum_{j\in \bN}\frac{|c_j|^2}{\lambda_{j}}<+\infty.
$$

	We equip $\mathcal{H}$ with the norm
	$$
	\left\|f_c\right\|_{\mathcal{H}}:=\biggl(\sum_{j\in \bN}\frac{|c_j|^2}{\lambda_{j}}\biggr)^{1/2}
	$$
	and inner product
	$$
	\langle f_c,f_{c'}\rangle _{\mathcal{H}} :=\sum_{j\in \bN}\frac{c_j c'_j}{\lambda_j}.
	$$
We shall show that $\cH_K$ is a vector-valued RKHS and $K$ happens to be its reproducing kernel. First notice for each $f\in L_{\rho}^2(X,Y)$,
	$$
		\left\|L_K^{1/2}(f)\right\|_{\mathcal{H}}=\sum_{j\in \bN}\left(\frac{|c_j|^2 \lambda_j}{\lambda_j}\right)^{\frac{1}{2}} =\bigl(\sum_{j\in \bN}|c_j|^2\bigr)^{\frac{1}{2}}=\bigl\|\sum_{j\in \bN}c_j\phi_j\bigr\|_{\rho}=\bigl\|P_{\Phi}f\bigr\|_{\rho} .
	$$
	Since the linear mapping $T:\mathcal{H} \to \ell^2(\bN)$ given by
	$$
	T(f_c):=(c_j/\lambda_{j}^{1/2}:j\in \bN),
	$$
	preserves norms, $\mathcal{H}$ is isomorphic to $\ell^2(\bN)$ and is hence a Hilbert space. We next prove that point evaluations are continuous on $\cH$ and $K$ is its reproducing kernel. Using the matrix norm induced by the vector norm $\|\cdot\|_2$, we see
	$$
	\begin{aligned}
	\left\| f_c(x)\right\|_2&= \biggl\| \sum_{j\in \bN}\frac{c_j}{\sqrt{\lambda_{j}}}\sqrt{\lambda_j} \phi_{j}(x)\biggr\|_2 \\
	&=\max\limits_{\ca\in Y
		\atop \mathop{\left\|\ca\right\|_2=1}}  \sum_{j\in \bN}\frac{c_j}{\sqrt{\lambda_{j}}}\sqrt{\lambda_j}\ca^{{T}} \phi_{j}(x) \\
	& \leq \max\limits_{\ca\in Y
		\atop \mathop{\left\|\ca\right\|_2=1}} \biggl(\sum_{j\in \bN}\frac{c^2_j}{\lambda_j}\biggr)^{1/2}  \biggl( \sum_{j\in \bN}\lambda_{j}|\ca^{{T}} \phi_{j}(x)|^2  \biggr)^{1/2} \\
	&= \biggl(\sum_{j\in \bN}\frac{c^2_j}{\lambda_j}\biggr)^{1/2}  \max\limits_{\ca\in Y
		\atop \mathop{\left\|\ca\right\|_2=1}} \biggl(\ca^{{T}}(\sum_{j\in \bN}\lambda_j \phi_{j}(x)\phi^{{T}}_{j}(x) )\ca \biggr)^{1/2} \\
	&=\biggl(\sum_{j\in \bN}\frac{c^2_j}{\lambda_j}\biggr)^{1/2} \biggl(\max\limits_{\ca\in Y
		\atop \mathop{\left\|\ca\right\|_2=1}} \ca^{T} K(x,x)\ca\biggr)^{1/2} \\
	&=\left\| f_c\right\|_{\mathcal{H}}  \sqrt{\left\|K(x,x)\right\|_2}
	\end{aligned}
	$$
	Thus, $\mathcal{H}$ is a vector-valued RKHS on $X$. Furthermore, for $\ba\in Y$
	$$
	K(x,\cdot)\ca=\sum_{j\in \bN}\lambda_{j}\phi_{j}(\cdot)\phi^T_j(x)\ca=\sum_{j\in \bN}u_j \phi_j:=f_u \in \mathcal{H},
	$$
	where $u_j=\lambda_{j}\phi_{j}^{{T}}(x)\ca $. It follows that $K(x,\cdot)\ca\in  \mathcal{H}$ for all $\ba\in Y$. Finally,  for all $ f_c\in \cH $, it holds
	$$
	\langle f_c,K(x,\cdot)\ca \rangle_{\mathcal{H}} =\sum_{j\in \bN}\frac{c_j u_j}{\lambda_{j}}=\sum_{j\in \bN}c_j \phi^{{T}}_{j}(x)\ca=\ca^{T} f_c(x),
	$$
	which verifies that $K$ is the reproducing kernel of $\mathcal{H}$. We conclude that $\mathcal{H}=\mathcal{H}_K$ as desired.
\end{proof}

The above result can be simplified if $K$ is also universal.
\begin{defn}
	{\bf     (universal kernels \cite{CMPY,CDTU,MXZ})}   A continuous matrix-valued reproducing kernel $K$ on $X$ is called {\bf universal } if
  $$
  \overline{span}\{K(x,\cdot)\cc :\cc\in Y ,x\in X\}=C(X,Y) ,
  $$
  where the closure is taken with respect to the norm on $C(X,Y)$. In other words, $K$ is universal if the linear span of $\{K(x,\cdot)\cc:\ x\in X,\ \cc\in Y\}$ is dense in $C(X,Y)$.
\end{defn}

 We have the following important corollary to Theorem \ref{Mercer-therom-result}.
\begin{coro}\label{universal-kernel}
	Assume the conditions in Theorem \ref{Mercer-therom-result}. If $K$ is a universal kernel on $X$ then
	$$
	\cH_{K}= L_K^{1/2}(L_{\rho}^2(X,Y))
	$$
	and $$
	\left\|L_K^{1/2}(f)\right\|_{K}=\left\|f \right\|_{\rho}, \ \ \forall f\in L_{\rho}^2(X,Y).
	$$
\end{coro}
\begin{proof}
	As the linear span of $\{K(x,\cdot)\cc:\ x\in X,\ \cc\in Y\}$ is dense in $C(X,Y)$, it is also dense in $L_{\rho}^2(X,Y)$. Therefore, there is no nontrivial function $g\in L^2_\rho(X,Y)$ such that $L_Kg=0$. Consequently, the orthogonal projection $P_\Phi$ is the identity operator on $L^2_\rho(X,Y)$. The result follows directly from Theorem \ref{Mercer-therom-result}.
\end{proof}

\section{Learning Rates of Multi-task Regularization Networks}
\setcounter{equation}{0}

We adopt the elegant idea in the classical paper \cite{SmaleZhou} to estimate the learning rate of multi-task regularization networks (\ref{mul-regnetworks}). In this section, we always assume that $X$ is a compact metric space, $\rho_X$ is non-degenerated on $X$, and $K$ is a continuous universal kernel on $X$.

\subsection{Error Decomposition}

By Theorem \ref{obvious-answer}, the minimizer of (\ref{mul-regnetworks}) is explicitly given by (\ref{representerequ}). To estimate the learning rate $\|f_{\bz,\lambda}-f_\rho\|_\rho$, we impose two more assumptions:
\begin{enumerate}
\item There exists some $\frac12<r\le 1$ such that $f_\rho\in L^r_K(L^2_\rho(X,Y))$.

\item The output data is almost surely bounded with respect to the probability measure $\rho$. Precisely, there exists some positive constant $M$ such that $\|y\|_\infty\le M$ almost surely on $Z=X\times Y$.
\end{enumerate}

The first assumption above ensures that $f_\rho\in \cH_K$.

\begin{prop}\label{f-rho-HK}
	If $ f_{\rho} \in L^{r}_K(L^{2}_{\rho}(X,Y))$  for some $r\ge \frac12$ then $f_\rho \in \mathcal{H}_K$.
\end{prop}
\begin{proof}
	Let $g=L_K^{-r}f_\rho$. Then $g \in L_{\rho}^{2}(X,Y)$ and $f_\rho=L_{K}^{r}g$. Since $K$ is universal, the eigenfunctions $\{\phi_n:n\in\bN\}$ of $L_K$ constitute an orthonormal basis of $L^2_\rho(X,Y)$. We factor $g$ under the basis as
	$$
	g=\sum_{j=1}^{\infty}d_j\phi_j
	$$
where $d=(d_j:j\in\bN)\in\ell^2(\bN)$ satisfies
$$
\sum_{j=1}^\infty |d_j|^2=\|g\|_\rho^2<+\infty.
$$
 By Definition \ref{square-root-operator} of $L_K^r$,
	$$
	f_\rho=\sum_{j=1}^{\infty}d_j\lambda_{j}^{r}\phi_j.
	$$
	By Theorem \ref{Mercer-therom-result}, we let
		$$
		c_j:=d_j\lambda_{j}^{r},\ \ j\in\bN
		$$	
	and compute that
	$$
		\|f_\rho\|_K=\sum_{j=1}^{\infty}\frac{|c_j|^2}{\lambda_j}=\sum_{j=1}^\infty |d_j|^2\lambda_j^{2r-1}<+\infty,
		$$
where we use $2r-1\ge0$ and the boundedness of $\{\lambda_j:j\in\bN\}$. The above equation shows that $f_\rho\in \cH_K$.
\end{proof}

By the above proposition and Theorem \ref{obvious-answer}, both $f_{\bz,\lambda}$ and $f_\rho$ belong to $\cH_K$ under our assumptions. We hence desire to bound $\|f_{\bz,\lambda}-f_\rho\|_\rho$ by $\|f_{\bz,\lambda}-f_\rho\|_K$. This can be done by the reproducing property (\ref{riesz-represent-theorem}).

Set
\begin{equation}\label{kappa}
\kappa := \max_{x\in X \atop 1\leq i,j \leq n} \sqrt{|K_{ij}(x,x)|},
\end{equation}
where $K_{ij}$ is the $ij$-th component function of the matrix-valued kernel $K$. It holds
\begin{equation}\label{matrixnormofK}
\|K(x,x)\|_2\le m\kappa.
\end{equation}

We have the following simple observation.
\begin{prop}\label{boundbyKnorm}
It holds for all $f\in\cH_K$ that
\begin{equation}\label{boundbyKnormeq}
\|f\|_\infty\le \kappa \|f\|_K,
\end{equation}
where
$$
\|f\|_{\infty}:=\sup_{x\in X}\|f(x)\|_\infty.
$$
Consequently,
\begin{equation}\label{boundbyKnormeq2}
\|f_{\bz,\lambda}-f_\rho\|_\rho\le \kappa\sqrt{m} \|f_{\bz,\lambda}-f_\rho\|_K.
\end{equation}
\end{prop}
\begin{proof}
Let $e_j,1\le j\le m$ be the standard basis of $\bR^m$. By the reproducing property, for all $f\in\cH_K$ and $x\in X$, the $i$-th component $f(x)_i$ of $f(x)$ satisfies
$$
|f(x)_i|=\bigl|\langle f(x),e_i\rangle_2\bigr|=\bigl|\langle f,K(x,\cdot)e_i\rangle_K\bigr|\le \|f\|_K\biggl(\langle K(x,x)e_i,e_i\rangle_2\biggr)^{1/2}=\|f\|_K\sqrt{K_{ii}(x,x)}\le \kappa \|f\|_K,
$$
which proves (\ref{boundbyKnormeq}). Consequently,
$$
\begin{aligned}
\|f_{\bz,\lambda}-f_\rho\|_\rho&=\left(\int_X\|f_{\bz,\lambda}-f_\rho\|_2^2d\rho_X\right)^{1/2}\le \left(\int_Xm\|f_{\bz,\lambda}-f_\rho\|_\infty^2d\rho_X\right)^{1/2}\\
&=\sqrt{m}\|f_{\bz,\lambda}-f_\rho\|_\infty\le \kappa\sqrt{m} \|f_{\bz,\lambda}-f_\rho\|_K,
\end{aligned}
$$
which proves (\ref{boundbyKnormeq2}).
\end{proof}

Therefore, our question boils down to bounding the error $\|f_{\bz,\lambda}-f_\rho\|_K$ in $\cH_K$. This is factored into two parts by the triangle inequality:
\begin{equation}\label{error-analysis}
\left\|f_{\bz,\lambda}-f_{\rho}\right\|_{K} \leq \left\|f_{\bz,\lambda}-f_{\lambda}\right\|_{K} + \left\|f_{\lambda}-f_{\rho}\right\|_{K},
\end{equation}
where $f_\lambda$ comes from the data-free model
\begin{equation}\label{f-lambda-frame}
	f_{\lambda}:=\arg\min_{f\in \mathcal{H}_{K}}\left\|f-f_{\rho}\right\|_{\rho}^{2}+\lambda\left\|f\right\|_{K}^2.
	\end{equation}
The factorization (\ref{error-analysis}) is standard in the theory of Cucker, Smale and Zhou \cite{CuckerSmale,CZ07}. The first term and second term in the right-hand side of (\ref{error-analysis}) are called the {\bf sampling error} and {\bf approximation error}, respectively.

\subsection{Approximation Error}

We start with the approximation error. We shall first derive an expression of $f_\rho$. To this end, we point out that $L_K^{1/2}$ is also a self-adjoint operator from $\cH_K$ to $\cH_K$. By Theorem \ref{Mercer-therom-result}, $\cH_K=L^{1/2}_K(L^2_\rho(X,Y))$ and for each $f\in L^2_\rho(X,Y)$,
$$
\left\|L_K^{1/2}f\right\|_K=\left\|f\right\|_{\rho}.
$$
It follows
\begin{equation}\label{K-to-L}
\langle L_K^{1/2}f,L_K^{1/2}g\rangle_K=\langle f,g\rangle_{\rho}, \ \ f,g\in L_{\rho}^2(X,Y).
\end{equation}

\begin{prop}
	The minimizer $f_\lambda$ of (\ref{f-lambda-frame}) exists and is unique. Moreover,
	\begin{equation}\label{f-lambda}
		f_\lambda=(L_K+\lambda I)^{-1}L_K f_\rho.
	\end{equation}
	where $I$ is identity operator on $\cH_K$.
\end{prop}
\begin{proof} Suppose $f_\lambda$ is a minimizer of (\ref{f-lambda-frame}). Then for every $f\in \cH_K$, the following function
$$
F(t)=\left\|f_\lambda+tf-f_{\rho}\right\|_{\rho}^{2}+\lambda\left\|f_\lambda +tf\right\|_{K}^2.
$$
attains its minimum at $t=0$. We compute
$$
F'(0)=2\langle f_\lambda-f_\rho,f \rangle_\rho +2\lambda \langle f_\lambda,f\rangle_K=0.
$$
By (\ref{K-to-L}),
$$
\begin{aligned}
0&=\langle L_{K}^{1/2}(f_\lambda -f_\rho), L_{K}^{1/2}f \rangle_K+\lambda \langle f_\lambda,f\rangle_K\\
&=\langle L_{K}(f_\lambda -f_\rho), f \rangle_K+\lambda \langle f_\lambda,f\rangle_K\\
&=\langle L_K(f_\lambda -f_\rho)+\lambda f_\lambda,f\rangle _K, \ \forall f\in\cH_K.
\end{aligned}
$$
It can be seen that the above condition is also sufficient for $f_\rho$ to be a minimizer of (\ref{f-lambda-frame}). Thus, the minimizer is uniquely given by
$$
f_\lambda=(L_K+\lambda I)^{-1}L_K f_\rho.
$$
\end{proof}

Similar arguments to those in \cite{SmaleZhou2} prove the following result on the approximation error.

\begin{theorem}\label{approximation-theorem}
   If $ f_{\rho} \in L^{r}_K(L^{2}_{\rho}(X,Y))$ for some $\frac12<r\le 1$ then
 \begin{equation}\label{approximation-inequation2}
     \left\|f_\lambda-f_\rho \right\|_K \leq \lambda^{r-1/2}\left\|L_{K}^{-r}f_\rho\right\|_\rho, \ \ \frac{1}{2} < r \leq 1.
 \end{equation}
\end{theorem}
\subsection{Sampling Error}

The sampling error will be estimated by the well-known Bennett inequality for vector-valued random variables.
\begin{lemma}\label{Bennett-corollary}\cite{Pinelis,SmaleZhou}
	Let $H$ be a Hilbert space and $\xi \in H$ be a random variable on $(Z,\rho)$. Suppose $\left\|\xi\right\|_H \leq \tilde{M}< \infty$ almost surely. Set $\sigma^{2}(\xi)=E(\left\|\xi\right\|_H^2)$. Given i.i.d. samples $\{{\xi_i}\}_{i=1}^{n}$ of $\xi$, for all $0<\delta<1$, with confidence $1-\delta$, it holds
	$$
	\left\|\frac{1}{n}\sum_{i=1}^{n}\xi_{i}-E(\xi)\right\|_H \leq \frac{2\tilde{M}log(2/\delta)}{n}+\sqrt{\frac{2\sigma^{2}(\xi)log(2/\delta)}{n}}.
	$$
\end{lemma}

By (\ref{representerequ}) and (\ref{f-lambda}), we decompose the sampling error as

\begin{equation}\label{sampling-error}
f_{\bz,\lambda}-f_{\lambda}=\left( \frac{1}{n}S_{\bx}^{*}S_{\bx}+\lambda I \right)^{-1} \left(\frac{1}{n}S_{\bx}^{*}y-\frac{1}{n}S_{\bx}^{*}S_{\bx}f_{\lambda}-\lambda f_{\lambda} \right).
\end{equation}
Observe
$$
\frac{1}{n}S_{\bx}^{*}y-\frac{1}{n}S_{\bx}^{*}S_{\bx}f_{\lambda}=\frac{1}{n}\sum_{i=1}^{n}K(x_i,\cdot)(y_i-f_{\lambda}(x_{i})).
$$
By (\ref{f-lambda}),
$$
\lambda f_{\lambda}=L_{K}(f_{\rho}-f_{\lambda}).
$$
Substituting the above equation into (\ref{sampling-error}), we obtain
\begin{equation}\label{sampling-error-scaling}
\left\|f_{\bz,\lambda}-f_{\lambda}\right\|_{K} \leq \frac{1}{\lambda}\left\|\frac{1}{n}\sum_{i=1}^{n}K(x_i,\cdot)(y_i-f_{\lambda}(x_{i}))-L_{K}(f_{\rho}-f_{\lambda})\right\|_{K}.
\end{equation}
We plan to bound the above quantity by Lemma \ref{Bennett-corollary}. To this end, we introduce the  vector-valued random variable
$$
\zeta(x,y) =K(x,\cdot)(y-f_{\lambda}(x)), \ \ x,y\in X\times Y.
$$

Important properties of this random variable are as follows.

\begin{theorem}\label{randomvariable}
It holds
$$
E(\zeta)=L_{K}(f_{\rho}-f_{\lambda})
$$
and
\begin{equation}\label{bound1}
\|\zeta\|_K\le \tilde{M}:=m\kappa (M+\|f_\lambda\|_{\infty}),\ \mbox{ almost surely}.
\end{equation}

\end{theorem}
\begin{proof} We first compute the expectation of $\zeta$:
$$
\begin{aligned}	
E(\zeta) &=\int_{X}K(x,\cdot)\int_{Y}(y-f_{\lambda}(x))d\rho(y|x)d\rho_{X}(x) \\
&=\int_{X}K(x,\cdot)\int_{Y}yd\rho(y|x)d\rho_{X}(x)-\int_{X}K(x,\cdot)f_{\lambda}(x)d\rho_{X}(x)\\
& =\int_{X}K(x,\cdot)f_{\rho}(x)d\rho_{X}(x)-\int_{X}K(x,\cdot)f_{\lambda}(x)d\rho_{X}(x)\\
&=L_{K}(f_{\rho}-f_{\lambda}).
\end{aligned}
$$
Then by the reproducing property (\ref{riesz-represent-theorem}),
$$
\|\zeta\|^2_{K}=\langle K(x,x)(y-f_\lambda(x)),y-f_\lambda(x) \rangle _2  \le m^2\kappa^2 (M+\left\|f_\lambda\right\|_{\infty})^2,
$$
which proves (\ref{bound1}).
\end{proof}

To continue, we make a few more observations.

\begin{lemma}\label{observations}
There hold
\begin{equation}\label{f-lambda-inequation}
 \left\|f_{\lambda}\right\|_{K} \leq \frac{ \left\|f_{\rho}\right\|_{\rho}}{\sqrt{\lambda}} \leq \frac{\sqrt{m}M}{\sqrt{\lambda}} ,
\end{equation}

\begin{equation}\label{sampling-inequation2}
\cE(f_\lambda)=\int_{Z}\left\|f_{\lambda}(x)-y\right\|_{2}^{2}d\rho \leq 2mM^{2},
\end{equation}
and
\begin{equation}\label{f-lambda-estimate}
\left\|f_\lambda\right\|_{\infty} \leq \frac{\kappa \sqrt{m}M}{\sqrt{\lambda}}.
\end{equation}
\end{lemma}
\begin{proof}
Since $f_\lambda$ is the minimizer of model (\ref{f-lambda-frame}), by choosing $f=0$ in the model, we have
\begin{equation}\label{sampling-inequation3}
\left\|f_{\lambda}-f_{\rho}\right\|_{\rho}^{2}+\lambda\left\|f_{\lambda}\right\|_{K}^{2} \leq \left\|f_{\rho}\right\|_{\rho}^{2}.
\end{equation}
As $\|y\|_\infty\le M$ almost surely,
\begin{equation}\label{sampling-inequation4}
\left\|f_{\rho}\right\|_{\rho}\le \sqrt{m}\|f_\rho\|_\infty\le \sqrt{m}\sup_{x\in X}\int_Y \|y\|_\infty d\rho(y|x)\le \sqrt{m}M.
\end{equation}
Combining the above two equations proves (\ref{f-lambda-inequation}).

For the second inequality, we let $f=0$ in (\ref{optimalrho}) to get
$$
\int_{Z}\left\| y \right\|_{2}^{2}d\rho -\int_{Z}\left\|f_{\rho}(x)-y\right\|_{2}^{2}d\rho=\left\|f_{\rho}\right\|_{\rho}^{2}.
$$
Thus,
$$
\cE(f_\rho)\leq \int_{Z}\left\| y \right\|_{2}^{2}d\rho \leq mM^{2}.
$$
We then let $ f=f_{\lambda} $ in (\ref{optimalrho}) to have by (\ref{sampling-inequation3}) and (\ref{sampling-inequation4}) that
$$
\cE(f_\lambda)=\cE(f_\rho)+\|f_\lambda-f_\rho\|_\rho^2\le  mM^{2}+ \|f_\lambda-f_\rho\|_\rho^2\leq 2mM^{2}.
$$
Finally, by (\ref{boundbyKnormeq}),
$$
\left\|f_\lambda\right\|_{\infty} \leq \kappa \left\|f_\lambda\right\|_{K}=\frac{\kappa \sqrt{m}M}{\sqrt{\lambda}},
$$
which proves (\ref{f-lambda-estimate}).
\end{proof}

We are in a position to estimate the sampling error.

\begin{theorem}\label{samplingerrorthm}
	With the assumptions at the beginning of this section and the further assumption that $\kappa\ge1$, for all $0  <  \delta < 1$ such that $\log(2/\delta)\ge1$, with confidence $1-\delta$, it holds
	\begin{equation}\label{sampling-anwser}
	\left\|f_{\bz,\lambda}-f_{\lambda}\right\|_{K} \leq \frac{6 m\kappa M\log(2/\delta)}{\sqrt{n}\lambda}.
	\end{equation}
\end{theorem}
\begin{proof}
By (\ref{sampling-error-scaling}),
\begin{equation}\label{lambda-estimateeq1}
\left\|f_{\bz,\lambda}-f_{\lambda}\right\|_{K} \leq \frac{\alpha}{\lambda}.
\end{equation}
where
$$
\alpha=\left\|\frac{1}{n}\sum_{i=1}^{n}K(x_i,\cdot)(y_i-f_{\lambda}(x_{i}))-L_{K}(f_{\rho}-f_{\lambda})\right\|_{K}.
$$
By (\ref{matrixnormofK}),
$$
\|\zeta\|^2_{K}=\langle K(x,x)(y-f_\lambda(x)),y-f_\lambda(x) \rangle _2\le \|K(x,x)\|_2\|y-f_\lambda(x)\|_2^2\le m\kappa\|y-f_\lambda(x)\|_2^2.
$$
Thus, the second moment of the random variable $\zeta$ satisfies
$$
\sigma^{2}(\zeta)=E(\left\|\zeta\right\|_K^{2}) \leq m\kappa\int_{Z}\left\|f_{\lambda}(x)-y\right\|^{2}_{2}d\rho=m\kappa\cE(f_\lambda).$$
By Lemma \ref{Bennett-corollary} and Theorem \ref{randomvariable}, with confidence $1-\delta$, it holds
\begin{equation}\label{lambda-estimate}
\begin{aligned}	
\alpha :=\left\|\frac{1}{n}\sum_{i=1}^{n}\zeta_{i}-E(\zeta)\right\|_{K}  &\leq \frac{2\tilde{M}\log(2/\delta)}{n}+ \sqrt {\frac{2\kappa m \log(2/\delta)\cE(f_\lambda)}{n}} \\
&=\frac{2m\kappa \log(2/\delta) (M+\left\|f_\lambda\right\|_{\infty})}{n}+\sqrt {\frac{2\kappa m \log(2/\delta)\cE(f_\lambda)}{n}}.
\end{aligned}
\end{equation}
By (\ref{sampling-inequation2}) and (\ref{f-lambda-estimate}),
$$
\alpha \leq \frac{2{m}\kappa M(1+\kappa \sqrt{m}/\sqrt{{\lambda}})\log(2/\delta)}{n}+2m\sqrt{\kappa}   M\sqrt{\frac{\log(2/\delta)}{n}}.
$$
We have two cases to discuss:
\begin{enumerate}
	\item If $\frac{\kappa \sqrt{m}}{\sqrt{n\lambda}} \leq \frac{1}{3log(2/\delta)}$ then
	$$
\begin{aligned}	
	\alpha&= \frac{2{m}\kappa M \log(2/\delta)}{n} +\frac{2 m \kappa M \log(2/\delta)}{\sqrt{n}}\frac{\kappa  \sqrt{m}}{\sqrt{n\lambda}}+2m\sqrt{\kappa} M\frac{\log(2/\delta)}{\sqrt{n}}\frac{1}{\sqrt{\log(2/\delta)}}\\
&\le \frac{6m\kappa M\log(2/\delta)}{\sqrt{n}}.
\end{aligned}
$$
By (\ref{lambda-estimateeq1}),
	$$
	\left\|f_{\bz,\lambda}-f_{\lambda}\right\|_{K} \leq \frac{6 m\kappa M\log(2/\delta)}{\sqrt{n}\lambda}.
	$$
	
	\item If $\frac{\kappa \sqrt{m}}{\sqrt{n\lambda}}   >   \frac{1}{3\log(2/\delta)}$ then
	$$
\frac{6 m\kappa M\log(2/\delta)}{\sqrt{n}\lambda}=\frac{6\sqrt{m} M\log(2/\delta)}{\sqrt{\lambda}}\frac{\kappa \sqrt{m}}{\sqrt{n\lambda}}  \geq \frac{2\sqrt{m}M}{\sqrt{\lambda}}.$$
	Letting $f=0$ in (\ref{mul-regnetworks}) yields
	$$
	\left\|f_{z,\lambda}\right\|_{K} \leq \sqrt{\frac{1}{\lambda}\,\frac{1}{ n}\sum_{i=1}^{n}\left\|y_i\right\|^{2}_{2}} \leq \frac{\sqrt{m}M}{\sqrt{\lambda}}.
	$$
	By (\ref{f-lambda-inequation}),
	$$
	\left\|f_{\lambda}\right\|_{K} \leq \frac{\sqrt{m}M}{\sqrt{\lambda}}.
	$$
	By the triangle inequality,
	$$
	\left\|f_{\bz,\lambda}-f_{\lambda}\right\|_{K} \leq \frac{2\sqrt{m}M}{\sqrt{\lambda}} \leq \frac{6m\kappa M\log(2/\delta)}{\lambda\sqrt{n}}.
	$$
\end{enumerate}
Therefore, we attain (\ref{sampling-anwser}) in both cases.
\end{proof}
\subsection{Ultimate Learning Rate}

We are ready to present the final learning rate for the multi-task regularization networks.

\begin{theorem}\label{ultimateratethm}
Let $\bz=\{(x_i,y_i):1\le i\le n\}$ be i.i.d. drawn from $Z=X\times Y$ according to an unknown probability measure $\rho$. Under the following assumptions:
\begin{itemize}
\item $X$ is a compact metric space and $\rho_X$ is non-degenerated on $X$,

\item the output data is almost surely bounded, that is, $\|y\|_\infty\le M$,

\item $K$ is a universal matrix-valued kernel on $X$,

\item  $f_\rho\in L^r_K(L^2_\rho(X,Y))$  for some $\frac12<r\le 1$,

\item the constant in (\ref{kappa}) satisfies $\kappa\ge 1$,
\end{itemize}
for all $0<\delta<1$ such that $\log(2/\delta)\ge1$, by choosing a regularization parameter $\lambda$ dependent on $n$ and $m$, we have with confidence $1-\delta$
$$
\left\|f_{\bz,\lambda}-f_{\rho}\right\|_{\rho} \leq 4\kappa\log(2/\delta) \left(3\kappa M\right)^{\frac{2r-1}{2r+1}} \, \left\|L_{K}^{-r}f_{\rho}\right\|_{\rho}^{\frac{2}{2r+1}}\  m^{\frac{6r-1}{4r+2}} \left(\frac{1}{n}\right)^{\frac{2r-1}{4r+2}}.
$$
\end{theorem}
\begin{proof}
By Theorem \ref{approximation-theorem} and Theorem \ref{samplingerrorthm}, we have upper bounds on the approximation error $\left\|f_{\lambda}-f_{\rho}\right\|_{K}$ and the sampling error $\left\|f_{\bz,\lambda}-f_{\lambda}\right\|_{K}$. Thus, by the triangle inequality,
\begin{equation}\label{total-error}
\begin{aligned}
\left\|f_{\bz,\lambda}-f_{\rho}\right\|_{K} &\leq \left\|f_{\bz,\lambda}-f_{\lambda}\right\|_{K} + \left\|f_{\lambda}-f_{\rho}\right\|_{K}\\
&\leq 2\log(2/\delta)\biggl(\frac{3m\kappa M}{\sqrt{n}\lambda}+{\lambda}^{r-\frac{1}{2}}\left\|L_{K}^{-r}f_{\rho}\right\|_{\rho}\biggr).
\end{aligned}
\end{equation}
We now choose an optimal regularization parameter as
$$
\lambda  = \left(\frac{3\kappa M}{\left\|L_{K}^{-r}f_{\rho}\right\|_{\rho}}\right)^{\frac{2}{2r+1}} \ \left(\frac{1}{n}\right)^\frac{1}{2r+1}\  m^{\frac{2}{2r+1}}
$$
to get
\begin{equation}\label{total-errorinHK}
\left\|f_{\bz,\lambda}-f_{\rho}\right\|_{K} \leq 4\log(2/\delta) \left(3\kappa M\right)^{\frac{2r-1}{2r+1}} \, \left\|L_{K}^{-r}f_{\rho}\right\|_{\rho}^{\frac{2}{2r+1}}\, \left(\frac{1}{n}\right)^{\frac{2r-1}{4r+2}}\, m^{\frac{2r-1}{2r+1}}.
\end{equation}
Finally we engage (\ref{boundbyKnormeq2}) to obtain the ultimate learning rate.
\end{proof}

When $m=1$, our estimate on $\|f_{\bz,\lambda}-f_{\rho}\|_{K}$ in (\ref{total-errorinHK}) is identical to that in Theorem 2 of \cite{SmaleZhou2}. Thus, the result above recovers the corresponding one for single-task learning when $m=1$. We remark that there are two crucial differences between our result and the classical results for single-task learning \cite{CuckerSmale,SmaleZhou}. Firstly, the regularization parameter depends both on the number of data and the number of tasks. Secondly, the final learning rate shows an dependence on the number of tasks. It reveals that as the number of tasks increases, the generalization ability of the regularization networks is indeed affected.

Finally, we make some discussions about the third assumption in Theorem \ref{ultimateratethm}:
$$
f_\rho\in L^r_K(L^2_\rho(X,Y)),\ \frac12<r\le 1.
$$
 That $f_\rho\in L^r_K(L^2_\rho(X,Y))$ for some $r>0$ is a standard assumption in learning theory \cite{CuckerSmale,SmaleZhou}. By Corollary \ref{universal-kernel} and Proposition \ref{f-rho-HK}, only $r\ge \frac12$ is able to ensure $f_\rho\in\cH_K$. Therefore, one cannot bound $\|f_{\bz,\lambda}-f_{\rho}\|_{\rho}$ by $\|f_{\bz,\lambda}-f_{\rho}\|_{K}$ if $0<r<\frac12$. Also, the quantity $\lambda^{r-\frac12}$ in the estimate of $\|f_\lambda-f_\rho\|_K$ in Theorem \ref{approximation-theorem} vanishes when $r=\frac12$. We conclude that the estimate methods in our paper do not apply to $0<r\le\frac12$. We shall investigate this case in a future study.

\end{document}